\documentclass[sigconf,nonacm]{acmart}

\usepackage{booktabs}
\usepackage{algorithm}
\usepackage{algpseudocode}
\usepackage{balance}
\usepackage{multirow}
\usepackage{amsmath}
\usepackage{amsfonts}
\usepackage{xspace}
\usepackage{url}
\usepackage{amsthm}
\usepackage{makecell}
\usepackage{comment}
\usepackage{afterpage}
\usepackage{float}
\usepackage{colortbl}

\newtheorem{definition}{Definition}
\newtheorem{theorem}{Theorem}
\newtheorem{corollary}{Corollary}
\newtheorem{proposition}{Proposition}
\newcommand{\modelname}{InRTL\xspace}

\newcommand{\lin}{\mathrm{lin}}
\algrenewcommand{\algorithmicrequire}{\textbf{Input:}}
\algrenewcommand{\algorithmicensure}{\textbf{Output:}}

\begin{document}

\title{InRTL: Effective Intra--Inter Interaction Learning for Relational Tables}

\author{Weichen Li}
\email{weichenli@sjtu.edu.cn}
\affiliation{%
  \institution{Shanghai Jiao Tong University}
  \city{Shanghai}
  \country{China}
}

\author{Ken Zhong}
\email{zhongken@sjtu.edu.cn}
\affiliation{%
  \institution{Shanghai Jiao Tong University}
  \city{Shanghai}
  \country{China}
}

\author{Zheng Wang}
\authornote{Zheng Wang is the corresponding author.}
\email{wzheng@sjtu.edu.cn}
\affiliation{%
  \institution{Shanghai Jiao Tong University}
  \city{Shanghai}
  \country{China}
}

\author{Li Pan}
\email{panli@sjtu.edu.cn}
\affiliation{%
  \institution{Shanghai Jiao Tong University}
  \city{Shanghai}
  \country{China}
}

\author{Jianhua Li}
\email{lijh888@sjtu.edu.cn}
\affiliation{%
  \institution{Shanghai Jiao Tong University}
  \city{Shanghai}
  \country{China}
}

\renewcommand{\shortauthors}{Weichen Li, Ken Zhong, Zheng Wang, Li Pan, \& Jianhua Li}

\begin{abstract}
Relational table learning has recently emerged as an important research direction for modeling multiple tables connected through primary key--foreign key (PK--FK) relationships. 
Despite recent advances, a principled modeling framework tailored to this task remains underexplored.
In this paper, we propose \textbf{I}ntra--I\textbf{n}ter \textbf{R}elational \textbf{T}able \textbf{L}earning (\textbf{\modelname}), a unified framework that explicitly models dependencies both within and across relational tables. 
Specifically, \modelname formalizes two complementary interaction patterns: \emph{intra-table interactions}, describing associations among rows within the same table, and \emph{inter-table interactions}, describing dependencies between rows across PK--FK-linked tables. 
To model these dependencies, we develop a column-aware table encoder to generate initial row representations, followed by Transformer-based self-attention and cross-attention modules for intra-table and inter-table learning, respectively. 
To further improve scalability, \modelname incorporates linearized attention and heterogeneous graph neural networks to simplify the self-attention and cross-attention operations. 
Extensive experiments on ten datasets covering 24 real-world tasks demonstrate the effectiveness of our approach. 
Code is available at \url{https://github.com/W1nterFloW/InRTL}.
\end{abstract}




\maketitle

\section{Introduction}\label{sec:introduction}
Tabular data is one of the most prevalent data modalities in real-world applications~\cite{cremaschi2024survey}.
It has a simple and well-defined structure, consisting of samples (rows) described by a consistent set of features (columns).
Owing to its structured format and strong interpretability, tabular data plays a central role across a wide range of domains, including medicine, finance, manufacturing, and many other application areas.

Traditional tabular learning methods primarily focus on single-table settings, aiming to capture complex interactions among table elements (like rows and columns) within the same table.
Early methods mainly relied on shallow models such as decision trees, random forests, and gradient boosting machines~\cite{friedman2001greedy}, while more recent approaches employ deep neural networks to capture complex nonlinear feature interactions~\cite{gorishniy2021revisiting}.
However, extending these methods to multi-table scenarios typically requires extensive manual feature engineering to denormalize multiple tables into a single flat table~\cite{kaggle2022survey}. 
This process is labor-intensive and often sacrifices the underlying relational structure, leading to performance bottlenecks.

Consequently, recent research has increasingly shifted toward relational table learning~\cite{zahradnik2023deep,robinson2024relbench,li2024rllm,chen2026tlsql}. 
This setting reflects modern relational databases, where data is distributed across multiple interconnected tables linked via primary-foreign key (PK-FK) relationships, forming rich semantic structures. 
Existing methods~\cite{chen2025relgnn,lachi2025boostingrelationaldeeplearning} typically combine tabular neural networks (TNNs)~\cite{borisov2022deep} with graph neural networks (GNNs)~\cite{kipf2017semisupervised}. 
TNNs encode each table independently to produce row-level embeddings capturing intra-table dependencies, which are then refined via GNN-based message passing to model inter-table relationships.
Although these approaches have shown some effectiveness, they largely rely on generic GNN message-passing mechanisms, and a principled relational modeling framework tailored to this task remains underexplored.

In this paper, we propose \textbf{I}ntra--I\textbf{n}ter \textbf{R}elational \textbf{T}able \textbf{L}earning (\textbf{\modelname}), a unified framework that explicitly models both intra-table and inter-table dependencies. 
Specifically, we define two complementary interaction patterns: (1) \emph{intra-table interactions}, which capture dependencies among rows within the same table, and (2) \emph{inter-table interactions}, which capture dependencies between rows across PK--FK-linked tables. 
Based on this formulation, we develop a column-aware table encoder to generate initial row representations, followed by Transformer-based self-attention and cross-attention modules~\cite{vaswani2017attention} for intra-table and inter-table modeling, respectively.

To improve scalability, \modelname further incorporates linearized attention and heterogeneous graph neural networks (HGNNs) to simplify the self-attention and cross-attention modules, respectively. 
Specifically, linearized attention reduces the computational cost of self-attention via a first-order Taylor approximation, while HGNNs reformulate cross-attention in terms of message passing to eliminate computations between rows without PK--FK relationships.
As a result, \modelname scales linearly with both the number of table rows and the number of inter-table PK--FK relations, preserving expressive power while enabling efficient learning on large relational datasets.

\modelname is simple yet theoretically grounded.
In contrast to traditional Transformer-based models, our method eliminates the need for positional encodings, or auxiliary training objectives, resulting in a compact and efficient learning pipeline.
Furthermore, our work provides new insights into the representational power of Transformers for large-scale relational tables: by connecting linear attention with classical HGNN message passing, we offer both theoretical and practical support for leveraging GNN-inspired techniques in relational tabular learning.

In summary, we make the following contributions:
\begin{itemize}
  \item We propose a novel relational tabular data learning method, \modelname, which explicitly defines and effectively captures both dependencies within and across tables.
  \item We theoretically analyze our method, including properties and error bounds of the linear attention mechanism, as well as the connection between Transformer and graph neural networks.
  \item We conduct extensive experiments on ten datasets from two well-established benchmarks, covering 24 real-world tasks, to validate the effectiveness of our proposed method. 
\end{itemize}

\section{Preliminaries}\label{sec:preliminaries}
In this section, we introduce some key definitions and essential preliminary concepts used throughout this paper.

\begin{figure*}[!t]
\centering
\includegraphics[width=0.8\textwidth]{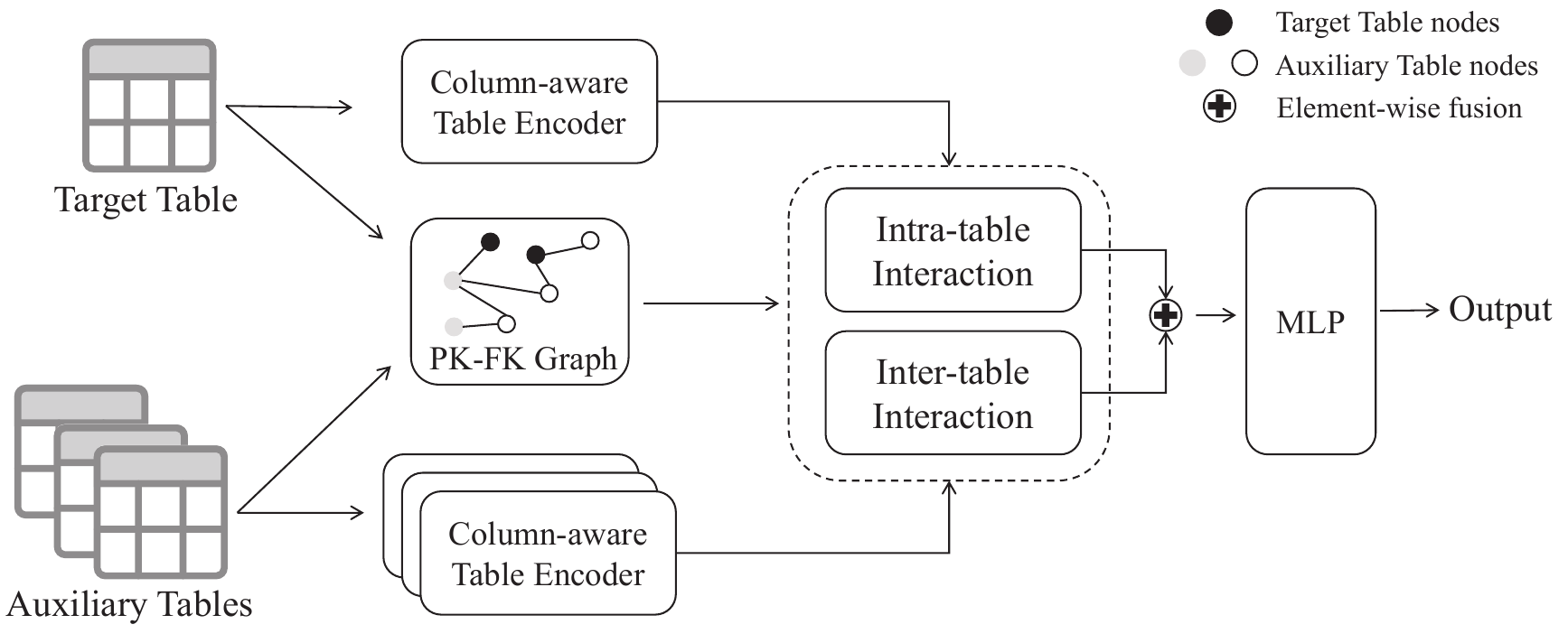}
\caption{Overview of {\modelname}, where each row in the target and auxiliary tables is treated as a node in the PK-FK graph.}
\Description{An overview figure illustrating the architecture of the proposed model.}
\label{fig:overview}
\end{figure*}

\subsection{Problem Definition}\label{preliminaries:problem_definition}
In this subsection, we formally define relational table data and the intra- and inter-table interactions within it.

\begin{definition}[\textit{Relational Table Data}]
Relational table data is defined as $(\mathcal{T}, \mathcal{K})$, where $\mathcal{T} = \{T_i\}_{i=1}^{|\mathcal{T}|}$ denotes a set of tables and
$\mathcal{K} \subseteq \mathcal{T} \times \mathcal{T}$ represents PK-FK relations among tables.
Typically, each table in $\mathcal{T}$ contains a primary key to uniquely identify each row.
If a table $T_i$ contains a foreign key that references the primary key of another table $T_j$, then there exists a relation $<T_i, T_j> \in \mathcal{K}$.
\end{definition}    

This definition characterizes multi-table data with primary and foreign key (PK-FK) structures, which are prevalent in practice. For instance, in the MovieLens dataset~\cite{harper2015movielens}, the user, movie, and rating tables are linked via user and movie IDs through PK-FK relations.

\begin{definition}[\textit{Intra-table Interaction}]
Intra-table Interaction refers to explicit or implicit associations between different rows within the same table, reflecting contextual and structural relations among rows.
\end{definition}
Intra-table interaction typically implies that there exists some form of contextual correlation between rows within the table. For example, in the user table of MovieLens, users of the same age group tend to exhibit similar preferences in movie ratings.

\begin{definition}[\textit{Inter-table Interaction}]
Given two tables $T_i$ and $T_j$ connected via a PK-FK relation, inter-table interaction refers to the row-level associations arising from this linkage between two tables. This relation reveals how $T_i$ and $T_j$ complement each other by providing contextual information.
\end{definition}
Inter-table interactions capture associations between rows of different tables connected by PK--FK relations.
For example, in the movies table of the MovieLens dataset, the movie ID can link to rows in the ratings table to obtain corresponding rating information. Modeling such inter-table interactions helps extract valuable contextual information from the related tables.

\subsection{Heterogeneous Graph Neural Networks}\label{preliminaries:hgnn}
Heterogeneous Graph Neural Networks (HGNNs)~\cite{zheng2022graph} are graph neural network frameworks specifically developed to learn representations over heterogeneous graphs.
Generally, an HGNN layer aims to compute the node representation $\vec{h_v}$ for each node $v$, formalized as:
\begin{equation}
    \vec{h_v} = \operatorname{UPD}_{\psi(v)} \bigl( \operatorname{AGG}_{r} \left( \left\{ \mathcal{N}_v^r: r \in \mathcal{R} \right\} \right) \bigr)
    \label{eq:hgnn}
\end{equation}
where $\mathcal{R}$ denotes the set of edge types, and $\mathcal{N}_v^r$ represents the neighbors of node $v$ connected via edges of type $r$. 
$\operatorname{AGG}_r(\cdot)$ is an edge-type-specific aggregation function, and $\operatorname{UPD}_{\psi(v)}(\cdot)$ is a node-type-specific update function for nodes of type $\psi(v)$.

\subsection{Transformer Architecture}\label{preliminaries:transformer_arch}
The Transformer~\cite{vaswani2017attention} consists of multiple stacked layers. Each layer has an attention module and a feed-forward network (FFN). Let $X \in \mathbb{R}^{n \times d}$ denote the input to the attention module, where $n$ is the sequence length and $d$ is the feature dimension. The input $X$ is linearly projected to query $Q$, key $K$, and value $V$ matrices using projection matrices \(W_Q, W_K, W_V\in \mathbb{R}^{d \times d}\). The attention is then computed as:
\begin{equation}
  \begin{gathered}
    Q = X W_Q,\ K = X W_K,\ V = X W_V \\
    A = \frac{Q K^\top}{\sqrt{d}},\quad
    \operatorname{Attn}(X)=\operatorname{Softmax}(A)V \\
  \end{gathered}
  \label{eq:orig-attn}
\end{equation}

For brevity, we only describe the self-attention with single-head here, extension to multi-head and cross-attention is straightforward.

\section{Methodology}\label{sec:methodology}
As shown in Figure~\ref{fig:overview}, the proposed {\modelname} method first employs a column-aware table encoder to extract feature representations from each table independently. Then, it adopts a Transformer-like architecture to simultaneously model both intra- and inter-table interactions to get representations. Finally, these two types of representations are integrated to generate the final output.

\subsection{Column-aware Table Encoder (ColATE)}\label{methodology:table_encoder}
This module is designed to learn high-dimensional and dense feature representations for each row in a single table. It consists of three components: pre-encoding module, feature-weighting module, and fusion module.

First, in the pre-encoding module, we apply a learnable parameterized mapping to embed each cell in the table into a high-dimensional space. After pre-encoding, a table with $n$ rows and $m$ columns is mapped to the same $d$-dimensional vector space, denoted as a collection of column matrices \(\{Z_1, Z_2, \dots, Z_m\} \) where $Z_i \in \mathbb{R}^{n \times d}$ represents the feature matrix of the $i$-th column.\footnote{
For notational simplicity, all representation/feature dimensions are denoted as \( d \) throughout the methodology section.} Pre-encoding is a crucial step in tabular representation learning. We refer to~\cite{gorishniy2022embeddings} for more details.

Next, in the feature-weighting module, we compute a separate importance weight for each column, since different tasks may rely unevenly on different columns. Specifically, for each column matrix $Z_i$, we perform average pooling over all its elements to produce a scalar importance score $s_i$. Then, a Softmax operation is applied to normalize the scores across all columns, resulting in the importance weights $\alpha_i$:
\begin{equation}
    \alpha_i = \frac{\exp(s_i)}{\sum_{j=1}^{m} \exp(s_j)}
\end{equation}

Finally, in the fusion module, we first apply these computed importance weights to re-scale each column matrix, thereby dynamically focusing on more informative features. After re-weighting, a deep residual network $\operatorname{ResNet}(\cdot)$~\cite{he2016identity} is employed to capture higher-order interactions among columns. The final aggregated representation $E \in \mathbb{R}^{n \times d}$ is computed as:
\begin{equation}
\label{eq:column_embedding}
E = \operatorname{ResNet}\bigl(\operatorname{Concat}(\alpha_1 Z_1,\;\dots,\;\alpha_m Z_m)\bigr)
\end{equation}
where $\operatorname{Concat}(\cdot)$ denotes concatenation along the column dimension.

\subsection{Modeling Intra- and Inter-table Interactions with Transformer
}\label{methodology:transformer-based method}
To simplify matters, we assume that there are only two tables with $n$ rows.\footnote{This assumption is made solely to simplify the presentation. As demonstrated in our experiments, our method is also applicable to scenarios involving multiple auxiliary tables with varying numbers of rows.}
Without loss of generality, we designate one table as the target table and the other as the auxiliary table.
The encoded representations of both tables by ColATE in Section~\ref{methodology:table_encoder} are denoted as \( E_\text{tgt}, E_\text{aux} \in \mathbb{R}^{n \times d} \).
We next describe how to model intra- and inter-table features for the target table using the classical Transformer architecture.

\subsubsection{Intra-table Interaction}
For the target table, let $H_{\text{intra}}^{(l)} \in \mathbb{R}^{n \times d}$ denote the intra-table representation at layer $l$ and initialize $H_{\text{intra}}^{(0)}$ with $E_{\text{tgt}}$.
We then apply self-attention to capture contextual dependencies within the target table.
Specifically, following the standard Transformer architecture with self-attention, the representation at layer $l{+}1$ is computed as:
\begin{equation}
    \hat{H}_{\text{intra}}^{(l)} = \operatorname{Attn}\bigl(\operatorname{LN}({H}_{\text{intra}}^{(l)})\bigr) \\
    \label{eq:inner-attn}
\end{equation}
\begin{equation}
    H_{\text{intra}}^{(l+1)} = \operatorname{FFN}\bigl(\operatorname{LN}(\hat{H}_{\text{intra}}^{(l)})\bigr) + \hat{H}_{\text{intra}}^{(l)}
    \label{eq:inner-forward}
\end{equation}
where $\operatorname{LN}(\cdot)$ denotes layer normalization.

Intuitively, self-attention allows each row to attend to all other rows within the same table, enabling the model to capture global contextual information within the table effectively.

\subsubsection{Inter-table Interaction}
Similarly, let $H_{\text{inter}}^{(l)} \in \mathbb{R}^{n \times d}$ denote the inter-table representation of the target table at layer $l$, and initialize $H_{\text{inter}}^{(0)}$ with $E_{\text{tgt}}$. We employ cross-attention (denoted by $\operatorname{CrossAttn}(\cdot)$) to model the interactions between the target and auxiliary tables. Specifically, $H_{\text{inter}}^{(l)}$ serves as the query source and $E_{\text{aux}}$ as the key and value sources. To mitigate noise from unrelated rows, we apply a binary $mask$ that identifies PK--FK-linked row pairs between target table and auxiliary table, 
masking unrelated entries with $-\infty$. The inter-table representation at the $l{+}1$ layer is computed via a Transformer block with cross-attention as:
\begin{equation}
\begin{aligned}
\hat{H}_{\text{inter}}^{(l)} =
&\operatorname{CrossAttn}\bigl(\operatorname{LN}({H}_{\text{inter}}^{(l)}),\\
&\operatorname{LN}(E_{\text{aux}}), mask
\bigr)
\end{aligned}
\label{eq:inter-attn}
\end{equation}    
\begin{equation}
    H_{\text{inter}}^{(l+1)} = \operatorname{FFN}\bigl(\operatorname{LN}(\hat{H}_{\text{inter}}^{(l)})\bigr) + \hat{H}_{\text{inter}}^{(l)}
\label{eq:inter-forward}
\end{equation}

Intuitively, cross-attention enables each row in the target table to attend selectively to PK--FK-related rows in the auxiliary table, effectively integrating relevant information while ignoring irrelevant rows.

\subsubsection{Combination of Two Interactions}
Finally, we fuse the intra-table and inter-table representations and apply a multilayer perceptron (MLP) to produce the final prediction.
Specifically, assuming the model has $L$ layers, the final prediction for the target table is given by:
\begin{equation}
    \operatorname{MLP}\bigl(\beta H_{\text{intra}}^{(L)} + (1 - \beta)H_{\text{inter}}^{(L)}\bigr)
    \label{eq:fusion}
\end{equation}
where $\beta$ is a learnable weighting coefficient that balances the contributions of intra- and inter-table representations.
The entire framework can then be optimized end-to-end using task-specific loss functions.

\noindent \subsubsection*{\textbf{Scalability Limitation}}
Adopting classical Transformer blocks for both intra-table (Eq.~\ref{eq:inner-attn}) and inter-table modeling (Eq.~\ref{eq:inter-attn}) requires pairwise attention among the $n$ rows in each case. 
Both involve computing an $n \times n$ attention matrix, resulting in $\mathcal{O}(n^2)$ time complexity~\cite{vaswani2017attention}, which limits scalability to large relational table datasets.

\subsection{Simplifying Transformer for Table Interactions}\label{methodology:simplification}
In this subsection, we propose simplification strategies for both intra-table and inter-table interactions to improve scalability to large tabular datasets.

\subsubsection{Intra-table Interaction Simplification}
To reduce the computational complexity of intra-table self-attention, we propose a linear approximation to the standard Transformer self-attention.

\begin{theorem}
Let $H_{\text{intra}}^{(l)} \in \mathbb{R}^{n \times d}$ denote the intra-table representation at the $l$-th layer, where $n$ is the number of rows and $d$ is the feature dimension.
Following the standard Transformer architecture (Eq.~\ref{eq:orig-attn}), define
\(
Q = H_{\text{intra}}^{(l)}W_Q,\ K = H_{\text{intra}}^{(l)}W_K,\ V = H_{\text{intra}}^{(l)}W_V.
\)
Assume that each row vector $\vec{q_i}$, $\vec{k_i}$, and $\vec{v_i}$ in $Q$, $K$, and $V$ is normalized such that
\(
\|\vec{q_i}\|_2 = \|\vec{k_i}\|_2 = \|\vec{v_i}\|_2 = 1.
\)
By applying a first-order Taylor expansion of the exponential function at zero within the Softmax operation of Eq.~\ref{eq:orig-attn}, the intra-table attention defined in Eq.~\ref{eq:inner-attn} becomes:
\begin{equation} \label{eq:linear-attn}
    \hat{H}_{\text{intra}}^{(l)} = 
    D^{-1}\bigl[J_{n}\,V + Q\,(K^\top V)\bigr]
\end{equation}
where $J_n = \vec{\mathbf{1}}_n \vec{\mathbf{1}}_n^\top \in \mathbb{R}^{n \times n}$ is an all-ones matrix, $\vec{\mathbf{1}}_n \in \mathbb{R}^n$ is a vector of all ones, and $D = \mathrm{Diag}\left(n + Q(K^\top \vec{\mathbf{1}}_n)\right)$ is a diagonal matrix.
\label{theo:linear-attn}
\end{theorem}

This theorem shows that self-attention in Transformer can be approximated with linear complexity.\footnote{Due to space limitations, all proofs of the formulas presented in this paper are provided in Appendix~\ref{appendix:proofs}.}
Substituting Eq.~\ref{eq:linear-attn} into Eq.~\ref{eq:inner-forward}, we derive the simplified formulation for modeling intra-table interactions:
\begin{equation}
    \begin{aligned}
        H_{\mathrm{intra}}^{(l+1)} &= \operatorname{FFN}\bigl(D^{-1}\bigl[J_{n}\,V + Q\,(K^\top V)\bigr]\bigr) \\
         &+ D^{-1}\bigl[J_{n}\,V + Q\,(K^\top V)\bigr]
    \end{aligned}
    \label{eq:intra-final}
\end{equation}

We now derive an upper bound for the approximation error of the linearized attention weights.
\begin{corollary}
\label{cor:error-bound}
Continue to use the notation $q$, $k$, and $n$ from Theorem~\ref{theo:linear-attn}.
Let $a_{ij}$ denote the original Transformer attention weight, and $\tilde{a}_{ij}$ denote the proposed linearized approximation:
\begin{equation}
    a_{ij}
    =\frac{\exp\bigl(\vec{q}_i^\top \vec{k}_j\bigr)}
          {\sum_{l=1}^n\exp\bigl(\vec{q_i}^\top \vec{k_l}\bigr)},
    \tilde a_{ij}
    =\frac{1 + \vec{q_i}^\top \vec{k_j}}
          {n + \sum_{l=1}^n \vec{q_i}^\top \vec{k_l}}
\label{eq:attn-score-and-linear-score}
\end{equation}
where \(i,j \in \{1,\dots,n\}\) are indices. For any $i$ and $j$, we have:
\[
    \bigl|\,a_{ij} - \tilde a_{ij}\bigr| = O(\epsilon^2)
\]
where $\epsilon$ is a constant such that $\max_{i,j} |\vec{q_i}^\top \vec{k_j}| \le \epsilon < 1$.
\end{corollary}

This corollary shows that the approximation error introduced in Eq.~\ref{eq:linear-attn} is second-order with respect to the query–key dot product. 
Since normalized query and key vectors tend to be nearly orthogonal in high-dimensional spaces, their dot products are typically small, and the resulting approximation error is negligible in practice.

\subsubsection{Inter-table Interaction Simplification}
Motivated by recent studies that investigate the connection between Transformers and graph neural networks~\cite{kim2022pure}, we adopt the HGNN framework as an alternative to the standard Transformer for modeling inter-table interactions.
\begin{proposition}
Let the bipartite graph be \(\mathcal{G} = (\mathcal{V}, \mathcal{U}, \mathcal{E})\), where \(\mathcal{V}\) and \(\mathcal{U}\) denote the sets of rows in the target and auxiliary tables (treated as nodes), and \(\mathcal{E}\) denotes the set of edges induced by PK--FK relationships.  
The inter-table interaction calculated in Eq.~\ref{eq:inter-attn} is equivalent to a single round of message passing over \(\mathcal{G}\), with attention weights serving as edge weights.
\label{props:inter-attn-gnn}
\end{proposition}
This proposition demonstrates that the cross-attention mechanism in Transformers can be effectively replaced by sparse HGNN layers, leading to substantial reductions in both time and space complexity. 
In particular, relational tables can be abstracted as a heterogeneous graph, where each PK--FK relation defines a distinct edge type. 
This allows the HGNN to model cross-table interactions in a computationally efficient manner.

Let the rows in the target and auxiliary tables correspond to heterogeneous graph nodes $\{v_i\}_{i=1}^{n}$, $\{u_j\}_{j=1}^{n}$ respectively, and $\vec{h}_{v_i}^{(l)}$, $\vec{h}_{u_j}^{(l)}$ denote the representations of the $i$-th row and $j$-th row treated as nodes in the target and auxiliary tables at layer $l$, initialized from the corresponding 
rows in $E_{\text{tgt}}$ and $E_{\text{aux}}$. Then, by reformulating Eq.\ref{eq:inter-attn} and Eq.\ref{eq:inter-forward} using Eq.~\ref{eq:hgnn}, we obtain:
\begin{equation}
  \vec{h}_{v_i}^{(l+1)}
    = \operatorname{UPD}_{\psi(v_i)}\Bigl(
        \vec{h}_{v_i}^{(l)},\,
        \operatorname{AGG}_r\{\vec{h}_{u_j}^{(l)} \mid u_j\in\mathcal{N}_{v_i}\}
      \Bigr)
  \label{eq:inter-hgnn}
\end{equation}
Consequently, for the target table, the final inter-table interaction representation is $H_{\text{inter}}^{(l+1)}=\bigl[\vec{h}_{v_1}^{(l+1)},\dots,\vec{h}_{v_n}^{(l+1)}\bigr]$.

\subsubsection{Combination of Two Interactions}
Similarly, after simplification, we compute both intra- and inter-table representations independently and fuse them using Eq.~\ref{eq:fusion}.
Furthermore, it is worth noting that this simplification enables our method to be combined with mini-batch sampling, allowing interactions to be modeled only on sampled subtables and subgraphs. This leads to efficient learning on large-scale relational table data.
For completeness, the overall training procedure of \modelname\ is summarized in Algorithm~\ref{alg:inrtl} in Appendix~\ref{appendix:Pseudocode}.

\section{Algorithm Analysis}
\label{sec:analysis}
In this section, we first analyze the time complexity of our method and then discuss its relation to existing approaches.

\subsection{Time Complexity Analysis}
\label{subsubsec:complexity}
The computational cost of our method arises from three main components: the column-aware table encoder, the linearized self-attention for intra-table interaction, and the HGNN message passing for inter-table interaction.
First, for the column-aware table encoder, the primary computational cost arises from the ResNet, which has a complexity of $\mathcal{O}(n m d^2)$, where $n$ and $m$ denote the number of rows and columns in the table, respectively.
Second, for intra-table interaction, the linear attention mechanism (i.e., Eq.~\ref{eq:linear-attn}) operates with a time complexity of $\mathcal{O}(nd^2)$.
Third, for inter-table interaction, the HGNN performs message passing on the heterogeneous graph with a single-layer complexity of $\mathcal{O}(ed)$, where $e$ represents the number of edges induced by PK--FK relations. 
In summary, the overall time complexity is $\mathcal{O}(n m d^2 + n d^2 + e d)$.
In practice, the number of rows and edges are typically much larger than the number of columns and feature dimension (i.e., $n,e \gg m,d$). 
Therefore, the overall complexity scales linearly with the number of table rows and the number of inter-table PK--FK relations.

\subsection{Relation to Prior Linear Attention Methods}
In this subsection, we compare our Taylor-expanded linear attention with several prior linearization methods.
To maintain notational consistency, we continue using the symbols introduced in Section~\ref{sec:methodology} throughout the discussion.

\subsubsection*{
\textbf{Relation to Linear Transformer~\cite{pmlr-v119-katharopoulos20a}}}
Linear Transformer linearizes the classical softmax attention (i.e., Eq.~\ref{eq:orig-attn}) using a kernel-based feature mapping $\phi(\cdot)$ applied independently to query and key representations:

\begin{equation}
\begin{aligned}
    \mathrm{Attn}(X) \approx & \ D^{-1}\phi(Q)\bigl(\phi(K)^\top V\bigr),\\
    D =\mathrm{Diag}&\bigl(\phi(Q)(\phi(K)^\top \vec{\mathbf{1}}_n)\bigr)
\end{aligned}
\end{equation}
The choice of $\phi(\cdot)$ (e.g., $\text{elu}(\cdot)+1$) is empirically driven without an exact theoretical derivation, making the approximation sensitive to kernel choice and input distributions.
In contrast, our method (i.e., Eq.~\ref{eq:linear-attn}) is based on a strong theoretical foundation, by directly applying a first-order Taylor expansion of the exponential term in the Softmax. Our formulation is deterministic, interpretable, and requires no additional hyper-parameters.
In fact, Linear Transformer only applies kernel $\phi(\cdot)$ to $Q$ and $K$ separately, if the kernel \(\phi(\cdot)\) is applied jointly to \(QK^\top\) such that \(\phi(QK^\top) = J_{n} + QK^\top\), it will directly yield our formulation (i.e., Eq.~\ref{eq:linear-attn}).

\subsubsection*{\textbf{Relation to Performer~\cite{choromanski2021rethinking}}}
Similar to Linear Transformer mentioned above, Performer also linearizes the softmax attention by using a kernel-based feature map. 
Specifically, it constructs a set of random features that serve as unbiased estimators of the Softmax kernel:
\begin{equation}
\mathrm{Attn}(X) \approx \mathbb{E}_{\vec{\boldsymbol{\omega}}}\Bigl[\phi_{\vec{\boldsymbol{\omega}}}(Q)\bigl(\phi_{\vec{\boldsymbol{\omega}}}(K)^{\top}V\bigr)\Bigr]
\end{equation}
where $\vec{\boldsymbol{\omega}}\sim\mathcal{N}(\mathbf{0},I_d)$ is a $d$-dimensional standard Gaussian, $\phi_{\vec{\boldsymbol{\omega}}}(\vec{\mathbf{x}})=\exp(\vec{\boldsymbol{\omega}}^\top \vec{\mathbf{x}}-\tfrac12\|\vec{\mathbf{x}}\|^2)$, $I_d$ is the $d\times d$ identity matrix, and $\mathbb{E}_{\vec{\boldsymbol{\omega}}}(\cdot)$ denotes the expectation over the $\vec{\boldsymbol{\omega}}$.
In practice, the expectation is approximated by averaging over a few samples of $\vec{\boldsymbol{\omega}}$, which introduces Monte Carlo variance. In contrast, Eq.~\ref{eq:linear-attn} in our method attains linear complexity without relying on random features. This is accomplished through a first-order Taylor expansion, resulting in a deterministic formulation that requires no additional hyper-parameters.

\begin{table}[!t]
\caption{Statistical analysis of the used benchmark datasets.}
\label{tab:datasets}
\centering
\begin{tabular}{llrr}
\toprule
Benchmark & Dataset & \#Tables & \#Rows \\
\midrule
\multirow{3}{*}{SJTUTables}
&TACM12K             & 4         & 97,774         \\
&TLF2K               & 3         & 101,773        \\
&TML1M               & 3         & 1,010,132      \\
\midrule
\multirow{7}{*}{RelBench}
&rel-f1              & 9         & 97,606          \\
&rel-trial           & 15        & 5,852,157      \\
&rel-avito           & 8         & 20,679,117     \\
&rel-amazon          & 3         & 24,291,489     \\
&rel-hm              & 3         & 33,265,846     \\
&rel-stack           & 7         & 38,109,828     \\
&rel-event           & 5         & 41,328,337     \\
\bottomrule
\end{tabular}
\end{table}

\begin{table}[!tbp]
  \caption{Classification results on SJTUTables (Accuracy \%, $\uparrow$).\protect\footnotemark}
  \label{tab:result_1}
  \centering
  \begin{tabular}{lccc|>{\columncolor{gray!15}}c}
    \toprule
    Model & TML1M & TLF2K & TACM12K & Avg. Rank\\
    \midrule
    LightGBM        & 26.66 & 35.52 & 35.70 & 5.7 \\
    TabPFN          & 25.20 & 43.20 & 34.10 & 5.3 \\
    FT-Transformer  & 28.60 & 15.80 & 14.80 & 8.3 \\
    TabNet          & 24.40 & 15.50 & 21.20 & 9.3 \\
    SAINT           & 27.90 & 13.70 & 16.50 & 9.0 \\
    Trompt          & 27.80 & 13.40 & 12.40 &  10.3 \\
    ExcelFormer     & 31.90 & 20.60 & 17.10 & 6.7 \\
    BRIDGE          & 36.20 & 42.20 & 25.60 & 3.3 \\
    RDL             & 35.60 & 28.70 & 19.20 & 5.3 \\
    RelGNN          & 27.12 & 20.31 & 18.19 & 8.0 \\
    LightRDL        & 24.94 & 39.17 & 35.77 & 5.7 \\
    \modelname      & \textbf{40.60} & \textbf{45.80} & \textbf{48.40} & \textbf{1.0}\\
    \bottomrule
  \end{tabular}
\end{table}
\footnotetext{The updated TLF2K dataset excludes artist embeddings, reducing accuracy from 50.10\% to 45.80\%.}

\begin{table*}[!t]
\caption{Binary classification results on RelBench (ROC-AUC \%, $\uparrow$).}
\label{tab:result_2}
\centering
\small
\begin{tabular}{llcccccccccc}
\toprule
Dataset & Task & LightGBM & TabPFN & FT-Trans & Trompt & ExcelFormer & RDL & RelGNN & BRIDGE & LightRDL & \modelname \\
\midrule
\multirow{2}{*}{rel-amazon} 
    & user-churn     & 52.22  & 60.51   & 52.30 & 53.10 & 52.09       & 70.42 & 65.44 & 70.51 & 69.17 & \textbf{70.53} \\
    & item-churn     & 62.54   & 63.48  & 64.15 & 59.82 & 62.50       & 82.81 & 82.64 & 79.93 & 81.20 & \textbf{82.83} \\
\midrule
\multirow{2}{*}{rel-avito} 
    & user-visits          & 53.05 & 65.34   & 52.10 & 56.11 & 55.01        & 65.99 & 65.47 & 66.16 & 62.39 & \textbf{66.25} \\
    & user-clicks          & 53.60 & 64.58   & 51.32 & 54.23 & 55.32        & 65.76 & 66.13 & 65.01 & 64.11 & \textbf{67.03} \\
\midrule
\multirow{2}{*}{rel-event} 
    & user-repeat   & 68.04 & 65.85 & 69.10 & 67.17 & 69.73        & 75.60 & 73.07 & 75.30 & 71.32 & \textbf{78.73} \\
    & user-ignore   & 79.93 & 83.24 & 75.69 & 77.31 & 77.92        & 75.14 & 79.11 & 76.16 & 76.30 & \textbf{85.80} \\
\midrule
\multirow{2}{*}{rel-f1}
    & driver-dnf      & 68.56 & 68.38 & 68.31 & 67.52 & 67.20      & 71.25 & 67.06  & 70.16 & 68.63 & \textbf{74.60} \\
    & driver-top3     & 73.92 & 72.18 & 73.62 & 73.17 & 75.22      & 77.23 & 80.86  & 77.13 & 79.16 & \textbf{83.21} \\
\midrule
\multirow{1}{*}{rel-hm} 
    & user-churn     & 55.21 & 69.75   & 50.13 & 54.19 & 53.95         & 70.11 & 68.10 & 69.33 & 67.13 & \textbf{70.32} \\
\midrule
\multirow{2}{*}{rel-stack} 
    & user-engagement   & 63.39  & 88.34    & 60.02 & 60.73  & 64.20         & 90.47 & 89.95 & 89.74 & 89.02 & \textbf{90.56} \\
    & user-badge        & 63.43  & 86.73    & 60.36 & 62.82  & 67.13         & 88.86 & 88.98 & 87.90 & 86.71 & \textbf{89.02} \\
\midrule
\multirow{1}{*}{rel-trial} 
    & study-outcome     & 70.09 & 67.88 & 62.39 & 65.73 & 67.31     & 69.02 & 69.16 & 68.98 & 68.15 & \textbf{70.20} \\
\midrule
\rowcolor{gray!15}
\multicolumn{2}{c}{Avg. Rank} & 6.92 & 6.00 & 8.75 & 8.25 & 7.58 & 3.33 & 3.92 & 4.08 & 5.17 & \textbf{1.00} \\
\bottomrule
\end{tabular}
\end{table*}

\subsubsection*{\textbf{Relation to Linformer~\cite{wang2020linformerselfattentionlinearcomplexity}}}
Linformer reduces the quadratic complexity of self-attention by 
introducing a low-rank projection on the key and value matrices:
\begin{equation}
    \mathrm{Attn}(X)
    \approx 
    \mathrm{softmax}\!\left(\frac{Q(\hat{K}K)^\top}{\sqrt{d}}\right)(\hat{V}V)
\end{equation}
where \(\hat{K}, \hat{V} \in \mathbb{R}^{d' \times n}\) are learnable projection matrices with \(d' \ll n\).
This formulation assumes that the attention matrix is inherently low-rank, 
thus approximating the Softmax weights through dimension reduction. However, a small projection dimension inevitably restricts the expressive power of the resulting attention scores, leading to potential information loss.
In contrast, as shown in Eq.~\ref{eq:attn-score-and-linear-score}, our Taylor-based approach directly linearizes the exponential kernel:
\(
e^{\vec{q_u}^\top \vec{k_v}} \approx 1 + \vec{q_u}^\top \vec{k_v}
\), yielding a closed-form expression, which achieves linear complexity \({O}(nd^2)\) without imposing any low-rank assumption.  
Therefore, our method preserves all pairwise interactions in full-rank space while maintaining efficiency comparable to Linformer.

In summary, prior linear attention methods achieve efficiency through various approximate strategies that may introduce uncontrolled uncertainty, such as heuristic kernel design, stochastic random features, or structural low-rank constraints.
Our method, in contrast, leverages deterministic Taylor linearization in inner-product space, improving stability and theoretical completeness.

\section{Experiments}
\label{sec:exps}
In this section, we conduct comprehensive experiments to evaluate our method against state-of-the-art baselines.

\subsection{Experimental Setup}
\subsubsection*{\textbf{Datasets}}
As summarized in Table~\ref{tab:datasets}, we conduct experiments on two public relational table learning benchmarks.
The first is SJTUTables~\cite{li2024rllm}, which includes three datasets covering domains such as online movies, online music, and academic paper citations.
Following the original benchmark configuration, we report accuracy as the evaluation metric for its three multi-class classification tasks.
The second benchmark, RelBench~\cite{robinson2024relbench}, includes seven datasets covering areas such as e-commerce, healthcare, and online communities.
This benchmark involves 21 tasks, including 12 binary classification tasks and 9 regression tasks.
In accordance with the benchmark guidelines, we adopt ROC-AUC for classification tasks and Mean Absolute Error (MAE) for regression tasks.
Additionally, for all experiments, we follow the standard data preprocessing and train/val/test split provided by each benchmark.

\subsubsection*{\textbf{Baselines}}
We compare our method with various baselines, categorized into three main groups.
The first group comprises traditional decision tree models, represented by LightGBM~\cite{ke2017lightgbm}, a gradient boosting framework known for its efficiency and strong performance on tabular data.
The second group includes deep learning-based single-table models: TabPFN~\cite{hollmann2022tabpfn}, FT-Transformer~\cite{gorishniy2021revisiting}, TabNet~\cite{arik2021tabnet}, SAINT~\cite{somepalli2021saint}, Trompt~\cite{pmlr-v202-chen23c}, and ExcelFormer~\cite{chen2023excelformer}.
These methods leverage attention mechanisms and advanced representation learning techniques to capture intra-table dependencies. 
The third group consists of relational table learning methods like BRIDGE~\cite{li2024rllm}, RDL~\cite{robinson2024relbench}, LightRDL~\cite{lachi2025boostingrelationaldeeplearning} and RelGNN~\cite{chen2025relgnn}.
The details about these baselines can be found in Appendix~\ref{appendix:baselines}.

\subsubsection*{\textbf{Implementation Details}}
We re-ran and carefully tuned all models using their official implementations. For RelGNN and LightRDL, which lacked complete training scripts, we re-implemented the missing components to ensure consistent evaluation. All reported results are averaged over 20 independent runs. 
Our experimental pipeline is built on the open-source relational table learning library relationLLM (rLLM)\footnote{\url{https://github.com/rllm-project/rllm}}, and additional implementation details are provided in Appendix~\ref{appendix:exp_details}.

\begin{table*}[t]
\caption{Regression results on RelBench (MAE, $\downarrow$).}
\label{tab:result_3}
\centering
\small
\begin{tabular}{llcccccccccc}
\toprule
Dataset & Task  & LightGBM & TabPFN & FT-Trans & Trompt & ExcelFormer & RDL & RelGNN & BRIDGE & LightRDL & \modelname \\
\midrule
\multirow{2}{*}{rel-amazon} 
    & user-ltv     & 16.783  & 19.413   & 16.930 & 16.606 & 17.001       & 14.313 & 16.783 & 15.374 & 14.310 & \textbf{14.250} \\
    & item-ltv     & 60.569  & 70.796   & 60.319 & 61.424 & 59.983       & 50.053 & 48.826 & \textbf{47.944} & 48.112 & 48.805 \\
\midrule
\multirow{1}{*}{rel-avito} 
    & ad-ctr          & 0.041 & 0.044   & 0.049 & 0.045 & 0.050        & 0.041 & 0.038 & 0.039 & 0.040 & \textbf{0.037} \\
\midrule
\multirow{1}{*}{rel-event} 
    & user-attendance   & 0.264 & 0.410 & 0.270 & 0.281 & 0.264        & 0.258 & 0.248 & 0.258 & 0.260 & \textbf{0.241} \\
\midrule
\multirow{1}{*}{rel-f1} 
    & driver-position     & 4.170 & 5.068 & 4.210 & 4.191 & 4.112      & 4.172 & 4.250  & 4.179 & 4.079 & \textbf{3.904} \\
\midrule
\multirow{1}{*}{rel-hm} 
    & item-sales     & 0.076 & 0.094   & 0.092 & 0.073 & 0.073         & 0.056 & 0.054 & 0.060 & \textbf{0.044} & 0.054 \\
\midrule
\multirow{1}{*}{rel-stack} 
    & post-votes   & 0.068  & 0.079     & 0.074 & 0.080  & 0.070         & 0.065 & 0.065 & 0.066 & 0.064 & \textbf{0.063} \\
\midrule
\multirow{2}{*}{rel-trial} 
    & study-adverse     & 44.011 & 45.609 & 44.490 & 45.397 & 45.552     & 44.473 & 44.461 & 44.773 & 43.910 & \textbf{43.522} \\
    & site-success     & 0.425 & 0.460   & 0.431 & 0.440   & 0.429      & 0.400  & 0.355  & 0.390   & 0.411 & \textbf{0.331} \\
\midrule
\rowcolor{gray!15}
\multicolumn{2}{c}{Avg. Rank} & 5.94 & 9.56 & 7.89 & 7.94 & 7.11 & 4.28 & 3.94 & 4.17 & 2.78 & \textbf{1.39} \\
\bottomrule
\end{tabular}
\end{table*}

\subsection{Performance Evaluation on Benchmarks}
Table~\ref{tab:result_1} presents the multi-class classification results on SJTUTables, while Table~\ref{tab:result_2} and Table~\ref{tab:result_3} report the binary classification and regression results on RelBench, respectively. Best values are in bold.
We summarize the following key observations from the results.

First, we find that relational table learning methods generally outperform single-table approaches, including both shallow and deep learning models. This highlights that leveraging multiple tables and their PK--FK relationships can significantly improve task performance, underscoring the strong potential of relational table learning.
Second, our method achieves consistently strong performance across all datasets and tasks.
This suggests that the proposed intra- and inter-table interaction mechanisms effectively capture multidimensional information to enhance performance.
This also demonstrates the generalizability of our approach and its potential for widespread application.
Third, we observe significant variation in the performance gains of our method across different datasets and tasks.
For example, substantial improvements are observed on TACM12K, rel-f1, and rel-event, whereas the gains are smaller on rel-amazon and rel-stack.
This variability likely stems from dataset-specific factors, such as differences in table size, sparsity, and PK--FK graph topology.
These insights motivate further investigation into how such data factors influence model performance, which presents a promising avenue for future research.
Last, on the SJTUTables dataset, we observe that two relational table learning methods (RDL and RelGNN) perform relatively poorly.
This may be because their implementations force each text column to use simple GloVe~\cite{pennington2014glove} embeddings independently.
In contrast, following BRIDGE, our method utilizes preprocessed embeddings from BERT~\cite{devlin2019bert} for entire rows of auxiliary text tables, enhancing textual understanding. Text encoder choice can materially affect relational table learning. Therefore, evaluating under different settings helps separate architectural gains from encoder gains. Further experiments on different text embedding choices, including GloVe and a lightweight BERT-family encoder (all-MiniLM-L6-v2~\cite{wang2021minilmv2multiheadselfattentionrelation}), are provided in Appendix~\ref{subsect_unified_embeddings}.

\subsection{Ablation Study}
This part is designed to assess the contribution of three major components in our method: the column-aware table encoder, the linear-attention-based intra-table interaction, and the HGNN-based inter-table interaction (denoted as ColATE, LIN-ATTN, and HGNN). 
We remove one component at a time from the full model and measure the performance to assess its individual contribution. 
Due to space limitations, we perform ablation studies on three representative datasets: rel-f1, TML1M, and rel-amazon, covering small-, medium-, and large-scale classification tasks.
Similar results are found on other datasets as well.

The results are presented in Table~\ref{tab:result_4}. We observe that each of the three components contributes to the overall performance of {\modelname}. Removing any module results in a noticeable performance degradation.
This demonstrates the effectiveness of the three core components in our method, underscoring the necessity of jointly modeling both intra- and inter-table interactions to fully exploit relational structures in tabular data.

\begin{table}[!t]
\caption{Ablation results, where TML1M is evaluated using accuracy \%, and the others are evaluated using ROC-AUC \%.}
  \label{tab:result_4}
  \centering
  \resizebox{\linewidth}{!}{
  \begin{tabular}{llcccc}
    \toprule
  \multirow{2}{*}{Dataset}
    & \multirow{2}{*}{Task}
    & \multirow{2}{*}{\modelname}
    & \multicolumn{3}{c}{Without} \\
  \cmidrule(lr){4-6}
    &       &       
    & ColATE  & LIN-ATTN   & HGNN   \\
  \midrule
    TML1M    & age-cls    
             & \textbf{40.60}     
             & 37.70  & 38.60  & 39.80     \\
    \addlinespace
    \multirow{2}{*}{rel-f1}
             & driver-dnf   
             & \textbf{74.60}     
             & 73.10  & 72.40  & 71.40     \\
             & driver-top3 
             & \textbf{83.90}     
             & 80.90  & 80.70  & 80.40     \\
    \addlinespace
    \multirow{2}{*}{rel-amazon}
             & user-churn   
             & \textbf{70.53}     
             & 70.49  & 70.45  & 56.07     \\
             & item-churn    
             & \textbf{82.80}     
             & 82.40  & 82.60  & 64.05     \\
    \bottomrule
  \end{tabular}
  }
\end{table}

\subsection{Hyper-parameter Analysis}
This subsection investigates the impact of two key hyper-parameters in our method. These experiments are all performed on the rel-f1 dataset under its default classification settings (driver-dnf and driver-top3). 
Consistent trends are observed across other datasets.

\begin{figure}[!t]
\centering
    \includegraphics[width=0.47\textwidth]{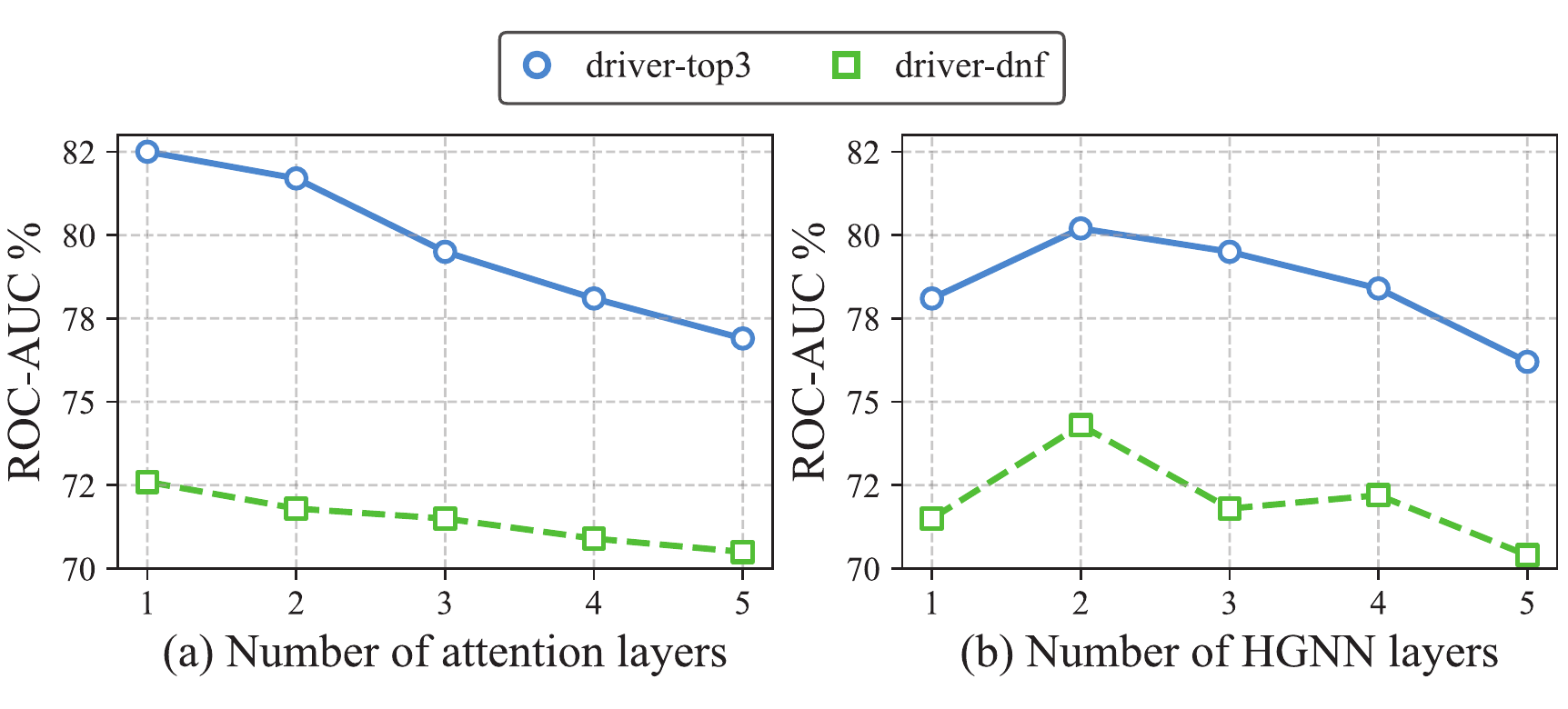}
\caption{Performance variation with respect to model depth.
(a) Linear attention layers for intra-table modeling.
(b) HGNN layers for inter-table modeling.
}
\Description{Performance variation with respect to model depth.}
\label{fig:depth}
\end{figure}   

\noindent\subsubsection*{\textbf{Effect of Model Depth}} 
We examine two depth-related factors: 
(i) the number of linear attention layers for intra-table interaction and 
(ii) the number of HGNN layers for inter-table interaction.
For each setting, we independently vary the number of layers, while keeping all other components fixed.
As shown in Figure~\ref{fig:depth}, we observe that increasing the number of linear attention layers consistently degrades performance.
We attribute this to the expressive power of multi-head attention, which captures sufficient intra-table interactions in the early stages, while deeper layers may introduce noise.
On the other hand, HGNN performance is non-monotonic: with increasing depth, the performance first rises then falls, peaking at two layers.
This aligns with well-known findings in GNNs.
Specifically, a single layer generally lacks sufficient receptive field to capture informative multi-hop dependencies across relational tables, while deeper networks risk exponential neighbor growth or over-smoothing~\cite{wang2024cluster}.
Therefore, using two HGNN layers achieves a good trade-off between expressiveness and efficiency.

\begin{figure}[!t]
\centering
\includegraphics[width=0.47\textwidth]{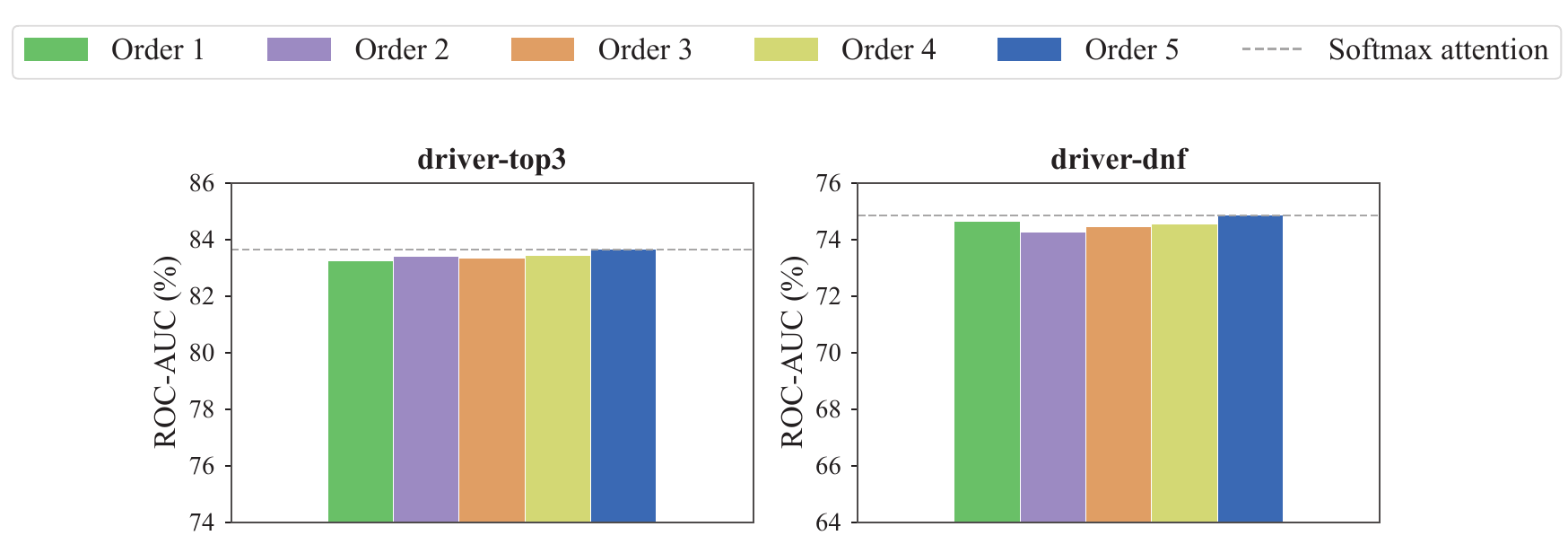}
\caption{ROC-AUC comparison of attention methods with different taylor series orders.}
\Description{ROC-AUC comparison of attention methods with different taylor series orders, with softmax attention as a reference line.}
\label{fig:attn_with_order}
\end{figure}

\noindent\subsubsection*{\textbf{Effect of Order Number in Taylor Expansion}}
We investigate the impact of the Taylor expansion order on attention optimization for modeling intra-table interactions.
In this experiment, all other components of \modelname are fixed, and we compare different attention mechanisms by varying the order of the Taylor expansion of the exponential function in Softmax.
As an important baseline, we also evaluate the original Transformer architecture (i.e., without any Taylor expansion).
As shown in Figure~\ref{fig:attn_with_order}, the performance gradually approaches that of the original Softmax attention as the Taylor expansion order increases.
We also find that the overall difference remains small, even with the first-order expansion.
Since only the first-order expansion has linear complexity (higher orders are all quadratic), we adopt it to achieve an effective balance between computational efficiency and modeling precision.

\subsection{Scalability Analysis}
To evaluate scalability, we synthetically generated several large graph datasets. In these datasets, each node is treated as a row in a table with 5 columns on average, and the total number of nodes ranges from 10K to 200K, with an average node degree of 5. We ran \modelname on these datasets and recorded the runtime of each stage for one epoch to assess its scalability.

\begin{figure}[!tbp]
\centering
\includegraphics[width=0.4\textwidth]{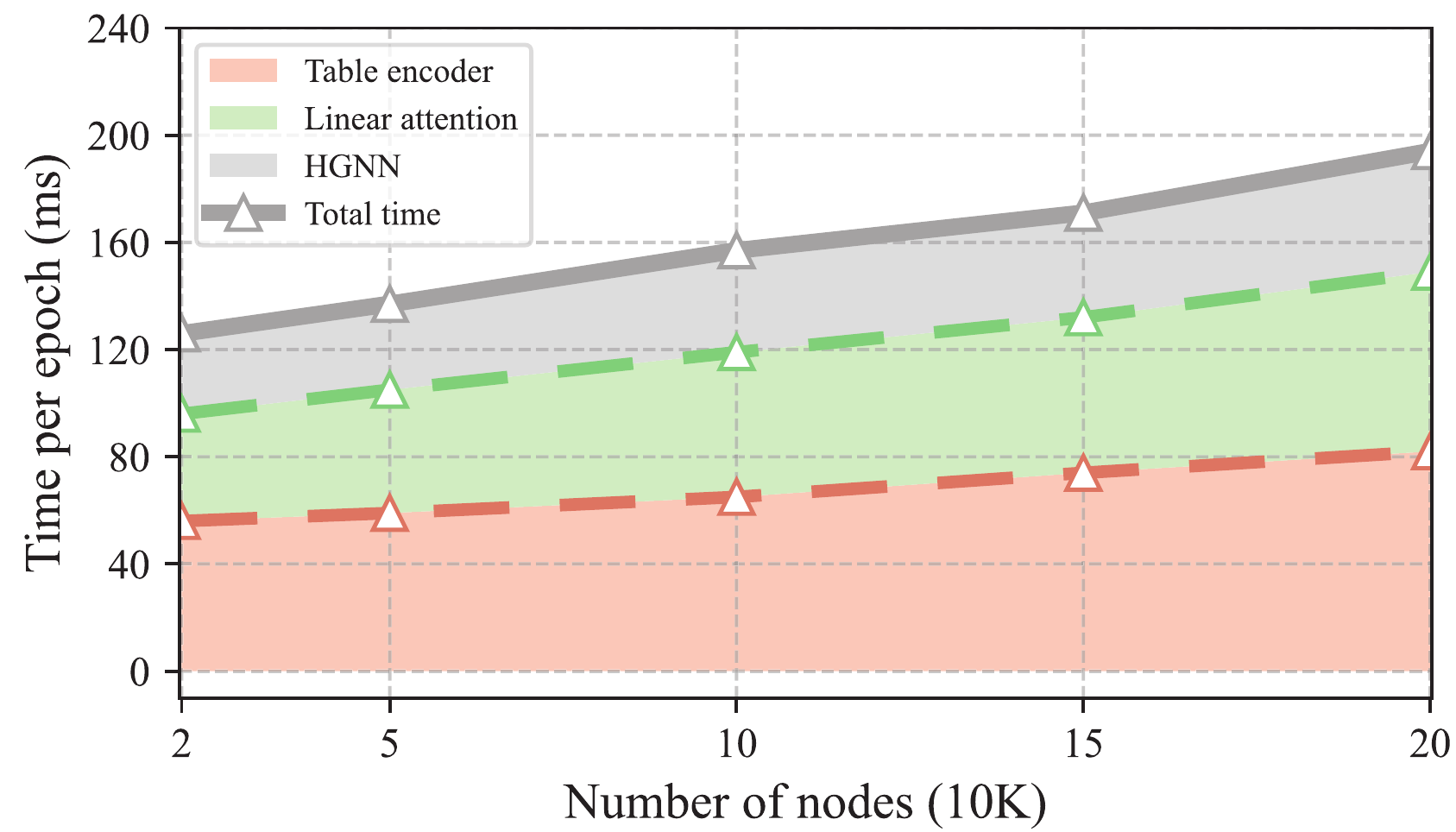} 
\caption{Scalability test of training time per epoch w.r.t. graph sizes (a.k.a. node numbers).}
\Description{Scalability test of training time per epoch w.r.t. graph sizes (a.k.a. node numbers).}
\label{fig:scalability}
\end{figure}

As shown in Figure~\ref{fig:scalability}, we can clearly see that the overall training time of \modelname grows approximately linearly with the number of nodes. The table encoder and linear attention module components dominate the runtime and exhibit similar growth patterns as the dataset size increases. This is because the table encoder part processes each feature column of every node, while the linear attention part operates on the node features, leading both components to scale roughly proportionally with the number of nodes and feature dimensions. In contrast, the HGNN component contributes a relatively small fraction of the total runtime, and its growth is moderate due to the fixed average node degree in the synthetic datasets. These experimental results confirm our theoretical expectations and demonstrate that \modelname scales efficiently to large datasets.

\section{Related Work}
\label{sec:relwork}
In this section, we review prior work on deep learning for single-table and multi-table tabular data. This section is organized into single-table modeling and multi-table modeling to clarify their assumptions and limitations.

\subsection{Deep Learning on Single Table}
Deep learning for single-table tabular data primarily focuses on learning expressive feature representations and capturing complex feature interactions without manual feature engineering~\cite{borisov2022deep}. 
A prominent line of research is based on attention mechanisms~\cite{vaswani2017attention}, where features are treated as tokens and modeled through self-attention. 
Representative methods such as TabTransformer~\cite{huang2020tabtransformertabulardatamodeling} and FT-Transformer~\cite{gorishniy2021revisiting} learn contextualized feature representations by modeling pairwise and higher-order feature dependencies within a table.
Another line of work emphasizes feature selection and interpretability. Models such as TabNet~\cite{arik2021tabnet} and ExcelFormer~\cite{chen2023excelformer} dynamically select and reweight informative features while suppressing irrelevant ones, thereby improving both predictive performance and interpretability.
More recently, prompt-inspired approaches such as Trompt~\cite{pmlr-v202-chen23c} incorporate feature semantics and instance-level importance into unified representations, bridging ideas from language modeling and tabular learning. Despite their effectiveness, these methods are fundamentally designed for isolated single tables and do not explicitly model relational dependencies across multiple tables.

\subsection{Deep Learning on Multiple Tables}
Recent efforts have increasingly explored deep learning on multi-table data, which can be broadly categorized into two lines of research.
The first line focuses on transfer learning across independent tables without explicit relational structure. These methods aim to capture shared representations or inductive biases across datasets to improve generalization to new tasks. Classical transfer learning approaches~\cite{wang2022transtab,kim2024carte} learn cross-dataset commonalities to enhance performance on downstream targets. More recently, tabular foundation model approaches such as TabPFN~\cite{hollmann2022tabpfn} pretrain on large collections of tabular datasets, enabling in-context learning and rapid adaptation to unseen tasks.
The second line focuses on relational tabular data~\cite{zahradnik2023deep}, where multiple tables are connected via foreign-key relationships. 
A common paradigm is to represent relational schemas as graphs, where tabular neural networks encode individual tables and graph neural networks capture inter-table dependencies~\cite{chen2025relgnn,lachi2025boostingrelationaldeeplearning}.
Recent benchmarks and systems such as RelBench~\cite{robinson2024relbench} and rLLM~\cite{li2024rllm} further demonstrate the effectiveness of this paradigm in real-world relational database scenarios.
Our proposed method falls into this second category. Unlike prior work, our approach explicitly defines and jointly models both intra-table and inter-table dependencies within a unified framework, enabling a clear and task-specific paradigm for relational table learning.

\section{Conclusion}
\label{sec:conclusion}
In this paper, we propose a novel relational table learning method \modelname.
Compared to traditional approaches, our method introduces a clear, task-specific objective designed to model relational tables effectively. Specifically, we formalize intra-table and inter-table interactions as explicit modeling objectives, implemented via Transformer-based self-attention and cross-attention modules.
To improve scalability, we propose two theoretically grounded simplification mechanisms that substantially reduce computational overhead, making \modelname\ practical for large-scale datasets.
Extensive experiments demonstrate that \modelname\ consistently outperforms state-of-the-art baselines on diverse relational table tasks.

\noindent\subsubsection*{\textbf{Limitations and Future Work}}
Our evaluation mainly focuses on well-curated relational databases with explicit table relationships. 
Extending relational table learning to open and heterogeneous data lakes remains challenging due to weak and noisy inter-table relationships~\cite{pan2026lakemlb}. 
Developing robust and scalable methods for such scenarios is an important direction for future research.

\bibliographystyle{ACM-Reference-Format}
\balance

\appendix
\section{Proofs}\label{appendix:proofs}
\subsection{Proof of Theorem~\ref{theo:linear-attn}}
\label{proof:linear-attn}
\begin{proof}
We begin by applying the first-order Taylor expansion of the exponential function at \(x=0\). For any real number \(x\), we have:
\begin{equation}\label{eq:exp_taylor}
    \exp(x) = 1 + x + \Omega(x)
\end{equation}
\begin{equation}\label{eq:remains}
     \Omega(x)=\frac{\exp(\xi)}{2}x^2,\quad \xi\in[0,x]
\end{equation}
where $\Omega(x)$ denotes the second-order remainder term of the Taylor expansion, and $\xi$ is a point in the interval $[0,x]$, related to the actual approximation error.

As the queries and keys in Transformer have been assumed to be normalized (i.e., $\|\vec{q_i}\|_2 = \|\vec{k_i}\|_2 = 1$), their inner product satisfies \(|\vec{q_u}^\top \vec{k_v}|\le1\) and is typically small in practice. 
Applying the first-order Taylor approximation to the exponential of this inner product, we can rewrite Eq.~\ref{eq:exp_taylor} as:
\begin{equation}
\exp(\,\vec{q_u}^\top \vec{k_v})
\approx  1 + \vec{q_u}^\top \vec{k_v}
\end{equation}

Using this approximation, the standard Transformer's attention weight (in Eq.~\ref{eq:orig-attn}) becomes: 
\begin{equation}\label{eq:approx_attn}
a_{uv}
= \frac{\exp(\vec{q_u}^\top \vec{k_v})}{\displaystyle\sum_{i=1}^n \exp(\vec{q_u}^\top \vec{k_i})}
\approx
\frac{1 + \vec{q_u}^\top \vec{k_v}}
     {\displaystyle\sum_{i=1}^n \bigl(1 + \vec{q_u}^\top \vec{k_i}\bigr)}
\end{equation}
Since \(|\vec{q_u}^\top \vec{k_v}|\le1\), we have  \(1 + \vec{q_u}^\top \vec{k_v} \ge 0\), ensuring that the approximated attention weights remain non-negative.

Then, we can express the numerator of the approximate attention weight in Eq.~\ref{eq:approx_attn} in matrix form:
\begin{equation}\label{eq:numerator}
A = [a_{uv}]_{u,v=1}^n = J_n + QK^\top
\end{equation}

Meanwhile, the denominator of the approximate attention weight in Eq.~\ref{eq:approx_attn} corresponds to a diagonal matrix formed by row sums of $A$:
\begin{equation}\label{eq:denominator}
D = \mathrm{Diag}(A\vec{\mathbf{1}}_n) = \mathrm{Diag}(n + Q(K^\top \vec{\mathbf{1}}_n))
\end{equation}

Combining Eq.~\ref{eq:numerator} and Eq.~\ref{eq:denominator}, the linearized attention matrix is given by \(D^{-1} A\). 
Substituting this into the attention output (in Eq.~\ref{eq:inner-attn}) and omitting layer normalization for simplicity, we obtain:
\begin{equation}
D^{-1}\,A\,V
= D^{-1}\bigl[J_nV + Q(K^\top V)\bigr]
\end{equation}

This completes the proof of the validity of the linear attention approximation.
\subsubsection*{\textbf{Time complexity}}
Computing \(Q(K^\top V)\) requires \(O(n d^2)\) operations. Computing \(J_n V\) by summing the columns of \(V\) and broadcasting the result requires only \(O(n d)\). Thus the overall complexity remains \(O(n d^2)\) and is substantially lower than the \(O(n^2 d)\) cost of standard Transformer attention. In the regime \(n\gg d\) the dominant cost scales as \(O(n)\).
\end{proof}

\subsection{Proof of Corollary~\ref{cor:error-bound}}
\label{proof:error-bound}
\begin{proof}
To derive the error bound for the approximation between standard softmax-based attention and our linearized form, we fix a row index \(u\) and have
\begin{equation}
x_i = \vec{q_u}^\top \vec{k_i}, 
\quad |x_i|\le\epsilon<1,
\quad i=1,\dots,n.
\label{eq:proof2-xi}
\end{equation}

Let \(a_i\) and \(\tilde{a}_i\) denote the standard softmax-based attention weight and its linearized approximation. By Eq.~\ref{eq:approx_attn} in~\ref{proof:linear-attn}, they can be expressed as:
\begin{equation}\label{eq:attns_of_xi}
a_i = \frac{\exp(x_i)}{\sum_{i=1}^n \exp(x_i)} ,
\quad
\tilde{a}_i = \frac{1 + x_i}{\sum_{i=1}^n (1 + x_i)}
\end{equation}
Let $D_{\exp}$ and $D_{\lin}$ denote the denominators for the standard and linearized attention weights:
\begin{equation}
D_{\exp} = \sum_{i=1}^n \exp(x_i)
\label{eq:proof2-orig-denominator}
\end{equation}
\begin{equation}
D_{\lin} = \sum_{i=1}^n (1 + x_i) = n + \sum_{i=1}^n x_i
\label{eq:proof2-linear-denominator}
\end{equation}

Since \(|x_i| \le \epsilon\), applying first-order Taylor approximation of the exponential
function (i.e., Eq.~\ref{eq:exp_taylor} and Eq.~\ref{eq:remains} in~\ref{proof:linear-attn}) to Eq.~\ref{eq:proof2-orig-denominator}, we obtain:
\begin{equation}
D_{\exp}
= \sum_{i=1}^n (1 + x_i + \Omega(x_i))
= D_{\lin} + \sum_{i=1}^n \Omega(x_i)
\label{eq:proof2-relation-denominator}
\end{equation}
\begin{equation}
|\Omega(x_i)| \le \tfrac{1}{2} \exp(\epsilon) \epsilon^2, \ 
\Bigl| \sum_{i=1}^n \Omega(x_i) \Bigr| \le \tfrac{n}{2} \exp(\epsilon) \epsilon^2
\label{eq:proof2-remains}
\end{equation}

Then, \(a_i\) and \(\tilde{a}_i\) in Eq.~\ref{eq:attns_of_xi} can be rewritten as:
\begin{equation}
a_i = \frac{\exp(x_i)}{D_{\exp}} = \frac{1 + x_i + \Omega(x_i)}{D_{\exp}},
\quad
\tilde{a}_i = \frac{1 + x_i}{D_{\lin}}
\end{equation}

Their difference can be decomposed as:
\begin{equation}
a_i - \tilde{a}_i
= \underbrace{\frac{\Omega(x_i)}{D_{\exp}}}_{(I)}
  + (1 + x_i)\underbrace{\left( \frac{1}{D_{\exp}} - \frac{1}{D_{\lin}} \right)}_{(II)}
\end{equation}

For term \((I)\), from Eq.~\ref{eq:proof2-relation-denominator}, it follows that \(|\Omega(x_i)| \le \tfrac{1}{2} \exp(\epsilon) \epsilon^2\). Substituting this into Eq.~\ref{eq:proof2-remains} yields:
\(
D_{\exp} \ge D_{\lin} - \biggl|\sum_i \Omega(x_i)\biggr| \ge n - \tfrac{n}{2} \exp(\epsilon) \epsilon^2.
\)
Accordingly, the following bound of term \((I)\) holds:
\begin{equation}\label{eq:O_I}
|(I)| \le \frac{\tfrac{1}{2}\exp(\epsilon) \epsilon^2}{n - \tfrac{n}{2}\exp(\epsilon) \epsilon^2} = O(\epsilon^2)
\end{equation}

For term \((II)\), observe that
\begin{equation}\label{eq:II}
\frac{1}{D_{\exp}} - \frac{1}{D_{\lin}}
= \frac{D_{\lin} - D_{\exp}}{D_{\exp}D_{\lin}}
= -\frac{\sum_i \Omega(x_i)}{D_{\exp} D_{\lin}}
\end{equation}
From Eq.~\ref{eq:proof2-remains}, we have established that \(\bigl|\sum_i \Omega(x_i)\bigr| \le \tfrac{n}{2} \exp(\epsilon) \epsilon^2\). Furthermore, applying the bound \(|x_i| \le \epsilon\) from Eq.~\ref{eq:proof2-xi} to Eq.~\ref{eq:proof2-linear-denominator}, we obtain
\(
D_{\lin} = n + \sum_i x_i \ge n(1 - \epsilon).
\)
In addition, as shown in the analysis of \((I)\), we have
\(
D_{\exp} \ge n - \tfrac{n}{2} \exp(\epsilon) \epsilon^2.
\)
Substituting the bounds on $\sum_i \Omega(x_i)$, $D_{\lin}$ and ${D_{\exp}}$ into Eq.~\ref{eq:II} gives:
\begin{equation}\label{eq:O_II}
|(II)| \le \frac{\tfrac{n}{2}\exp(\epsilon) \epsilon^2}{n(1-\epsilon) \cdot (n -\tfrac{n}{2} \exp(\epsilon) \epsilon^2)} = O(\epsilon^2)
\end{equation}

Therefore, combining both Eq.~\ref{eq:O_I} and Eq.~\ref{eq:O_II}, we conclude:
\begin{equation}
\bigl| a_i - \tilde{a}_i \bigr| = O(\epsilon^2)
\end{equation}

This completes the proof.
\end{proof}

\subsection{Proof of Proposition~\ref{props:inter-attn-gnn}}
\begin{proof}
Consider the bipartite graph \(\mathcal{G} = (\mathcal{V}, \mathcal{U}, \mathcal{E})\), which contains exactly two node types, \(\psi(v_i)\) and \(\psi(u_i)\), and a single edge type \(r\). We only consider target nodes here.

Based on the edge set \(\mathcal{E}\), we define a masking matrix \(\mathrm{mask} \in \mathbb{R}^{n \times n}\) as:
\begin{equation}\label{eq:mask}
\mathrm{mask}_{ij} =
\begin{cases}
0, & \text{if } (u_i, v_j) \in \mathcal{E}, \\
-\infty, & \text{otherwise}.
\end{cases}
\end{equation}

For each target node \(v_i\), the aggregation operator \(\operatorname{AGG}_r(\cdot)\) defined in Eq.~\ref{eq:hgnn} collects features from all neighboring auxiliary nodes. The aggregated message at layer \(l\), denoted by \(msg_{v_i}^{(l)}\), is computed as:
\begin{equation}\label{eq:agg}
    \begin{aligned}
        msg_{v_i}^{(l)} &= \operatorname{AGG}_r\left( \left\{\vec{h}_{u_j}^{(l)} \mid u_j \in \mathcal{N}_{v_i}^r \right\} \right) \\
        &=\sum_{u_j \in \mathcal{N}_{v_i}^r} a_{v_i u_j} \vec{h}_{u_j}^{(l)}
    \end{aligned}
\end{equation}
where the weight is defined as
\[
a_{v_i u_j} = \frac{\exp\left(\vec{h}_{v_i}^{\top} \vec{h}_{u_j}\right)}
                   {\sum_{u_k \in \mathcal{N}_{v_i}^r} \exp\left(\vec{h}_{v_i}^{\top} \vec{h}_{u_k}\right)}
\]

We further define the update operator \(\operatorname{UPD}_{\psi(v_i)}\) in Eq.~\ref{eq:hgnn} as the identity mapping. The updated representation of the target node \(v_i\) at layer \(l{+}1\), denoted \(h_{v_i}^{(l+1)}\), is given by:
\begin{equation}\label{eq:upd}
    \vec{h}_{v_i}^{(l+1)} = \operatorname{UPD}_{\psi(v_i)}(\vec{h}_{v_i}^{(l)}, msg_{v_i}^{(l)}) = msg_{v_i}^{(l)}
\end{equation}

Therefore, after one round of message passing using Eq.~\ref{eq:agg} and Eq.~\ref{eq:upd}, the representation of node \(v_i\) is updated as:
\begin{equation}
    \vec{h}_{v_i}^{(l+1)} = \sum_{u_j \in \mathcal{N}_{v_i}^r} a_{v_i u_j} \vec{h}_{u_j}^{(l)}
    \label{eq:proof3-result}
\end{equation}

Let \(Q =\bigl[\vec{h}_{v_1}^{(l)},\dots,\vec{h}_{v_n}^{(l)}\bigr]\), \(K = V = \bigl[\vec{h}_{u_1}^{(l)},\dots,\vec{h}_{u_n}^{(l)}\bigr]\). Then, using the mask defined in Eq.~\ref{eq:mask}, Eq.~\ref{eq:proof3-result} can be reformulated in matrix form as:
\begin{equation}
    \mathrm{Softmax}\!\left(Q K^\top + \mathrm{mask}\right) V
\end{equation}
which is exactly equivalent to the Transformer cross-attention mechanism used in Eq.~\ref{eq:inter-attn}, omitting layer normalization and temperature scaling (i.e., division by \(\sqrt{d}\) ).
\end{proof}

\section{Baselines}
\label{appendix:baselines}
To conduct a rigorous and comprehensive empirical evaluation, we compare our method against a diverse set of competitive and widely recognized baselines. These baselines are organized into three categories: \emph{traditional decision tree models}, \emph{deep learning-based single table learning methods}, and \emph{relational table learning methods}.

\begin{itemize}

\item \textbf{LightGBM}~\cite{ke2017lightgbm} is a scalable tree-based gradient boosting framework optimized for efficient tabular data learning.

\item \textbf{TabPFN}~\cite{hollmann2022tabpfn} is a tabular foundation model based on transformers that performs Bayesian-style in-context prediction using the Prior-Data Fitted Network paradigm.

\item \textbf{FT-Transformer}~\cite{gorishniy2021revisiting} adapts the standard Transformer architecture for tabular data by projecting heterogeneous features into a shared embedding space for self-attention.

\item \textbf{SAINT}~\cite{somepalli2021saint} enhances tabular learning by applying joint attention across both rows and columns alongside contrastive self-supervised pre-training.

\item \textbf{Trompt}~\cite{pmlr-v202-chen23c} disentangles intrinsic column semantics from instance-conditioned feature importance to flexibly model tabular data.

\item \textbf{ExcelFormer}~\cite{chen2023excelformer} is a robust single-table model combining semi-permeable attention, data augmentation, and a gated attentive feed-forward module.

\item \textbf{RDL}~\cite{robinson2024relbench} enables end-to-end relational learning by encoding a database as a heterogeneous graph and integrating neural table encoders with GNNs.

\item \textbf{BRIDGE}~\cite{li2024rllm} simplifies relational table learning as a GCN-based baseline that focuses on a single relational table while disregarding inter-table heterogeneity.

\item \textbf{RelGNN}~\cite{chen2025relgnn} constructs atomic relational message routes based on primary-foreign key links using a hybrid message-passing and attention mechanism.

\item \textbf{LightRDL}~\cite{lachi2025boostingrelationaldeeplearning} extends RDL to capture temporal dynamics by constructing snapshotted relational graphs and applying heterogeneous GraphSAGE propagation.

\end{itemize}

\section{Implementation details}\label{appendix:exp_details}
In this section, we present experimental details, including data preparation, baselines details, model details, and hyper-parameter settings, to facilitate the reproducibility of our results.

\subsection{Data Preparation}
For the SJTUTables benchmark, we follow the original graph construction procedure proposed in its benchmark. Given the dataset's modest size, we employ full-batch training by feeding the entire graph into the model during each iteration. For RelBench, we adopt the data processing pipeline outlined in the original paper, where relational tables are transformed into heterogeneous temporal graphs. Owing to the large scale of the datasets, we apply a mini-batch training strategy based on temporal neighbor sampling. Specifically, we apply temporal neighbor sampling proposed in the original RelBench paper to generate subtables and subgraphs, compute loss, and update parameters on each sampled subgraph.

Each dataset is split into training, validation, and test sets, strictly following the splits defined in their original papers. We train models on the training set, evaluate performance on the validation set, and report final metrics on the test set.  We employ cross-entropy loss for classification tasks and L1 loss for regression tasks.
All models are trained on a single Nvidia RTX A6000 GPU with 48GB of memory.
All reported results are averaged over 20 independent runs.

\subsection{Baseline Details}
Our baselines include traditional decision tree models, deep learning-based single-table models, and relational table learning methods.
For decision tree and single-table models, which cannot natively handle multi-relational inputs, we adopt a Flattening strategy. This applies a sequence of LEFT JOIN operations on the target table to incorporate information from all auxiliary tables, resulting in a single, feature-rich table used for training and evaluation.

For relational table learning methods, we follow their original design constraints. In particular, BRIDGE supports only unary relationships between two tables; thus, its graph construction includes only one auxiliary table connected to the target table. Other relational baselines are allowed to use all available auxiliary tables.

To ensure fairness, we re-implemented all baselines based on the official code or their original papers and performed a comprehensive grid search over key hyper-parameters (see Appendix~\ref{subsect_hps} for details).
For methods with incomplete or non-runnable public implementations (e.g., RelGNN and LightRDL, whose training scripts were not fully provided), we re-implemented the missing training components to enable consistent evaluation.
We report the best performance for each baseline under our experimental setup.

\subsection{Our Model Details}
Our model comprises three key components: a column-aware table encoder, a linear attention module for intra-table interactions, and a HGNN module for inter-table interactions. We describe each component below. 

The column-aware table encoder follows the suggestion in RelBench, using specialized pre-encoding strategies for different column types. Categorical features are embedded using a lookup table that maps discrete values into a shared embedding matrix. Column offsets are added to distinguish between different columns, enabling the model to generate dense, column-specific representations. For textual features, we apply pretrained embeddings such as GloVe~\cite{pennington2014glove} or transformer-based language models like MiniLM~\cite{wang2021minilmv2multiheadselfattentionrelation,allMiniLML6v2_2021} to obtain vector representations. The number of ResNet layers used in the fusion module is configurable.

The linear attention module extends Eq.~\ref{eq:intra-final} into a multi-head variant during training. Each attention head projects inputs into a separate subspace, allowing the model to capture diverse interaction patterns in parallel. This design improves both the representational capacity and the robustness of intra-table modeling.

The HGNN module models inter-table relationships based on PK–FK constraints. These are represented as undirected edges in a heterogeneous graph. The aggregation function in Eq.~\ref{eq:inter-hgnn} is implemented as an element-wise summation over neighboring node features. The update function applies a linear transformation to the aggregated result, combined with the node's own representation.

\subsection{Comparison with Unified Text Embeddings}
\label{subsect_unified_embeddings}
To remove the potential confounding effect of text initialization, we reran SJTUTables comparisons using unified text embeddings for the main methods InRTL, RDL, and RelGNN.

Specifically, as shown in Table~\ref{tab:unified_embedding}, we report results under two shared embedding settings: (1) GloVe, and (2) lightweight BERT-family encoder all-MiniLM-L6-v2.
As shown in Table~\ref{tab:unified_embedding}, under both unified settings, InRTL consistently remains the best on all three SJTUTables datasets, indicating that the performance gain is not solely due to text embedding choice.

\begin{table}[!t]
\centering
\caption{Unified text embeddings comparison on SJTUTables (Accuracy \%, $\uparrow$).}
\label{tab:unified_embedding}
\begin{tabular}{l l c c c}
\toprule
Embedding & Method & TML1M & TLF2K & TACM12K \\
\midrule
GloVe & InRTL & \textbf{39.83} & \textbf{49.67} & \textbf{46.80} \\
GloVe & RDL & 35.60 & 28.65 & 19.50 \\
GloVe & RelGNN & 28.43 & 21.34 & 20.20 \\
\midrule
all-MiniLM-L6-v2 & InRTL & \textbf{40.30} & \textbf{50.33} & \textbf{48.17} \\
all-MiniLM-L6-v2 & RDL & 36.40 & 31.67 & 19.80 \\
all-MiniLM-L6-v2 & RelGNN & 28.60 & 24.70 & 20.43 \\
\bottomrule
\end{tabular}
\end{table}

\subsection{Hyper-parameters Settings}
\label{subsect_hps}
We tune all model hyper-parameters based on validation performance.
Unless otherwise stated, the hyper-parameters are selected from the following candidate values using exhaustive grid search.

\begin{itemize}
    \item batch size: \{128, 256, 512\}
    \item learning rate: \{1e-3, 5e-3, 1e-2\}
    \item hidden dimension: \{32, 64, 128, 256, 512\}
    \item weight decay: \{1e-5, 1e-4, 1e-3, 1e-2\}
\end{itemize}

The search spaces for other model-specific hyper-parameters are listed below:

\begin{itemize}
    \item LightGBM: max depth $\in$ \{3, 4, 5, 6, 7, 8, 9, 10\}, number of leaves $\in$ \{2, 32, 64, 128, 256, 512\}.
    \item FT-Transformer: number of heads $\in$ \{4, 8, 16\}, number of layers $\in$ \{2, 3, 4, 6\}.
    \item TabNet: number of steps $\in$ \{3, 4, 5, 6\}, $\gamma$ $\in$ \{1.0, 1.5, 2.0\}.
    \item Trompt: number of Trompt Cells $\in$ \{1, 3, 6, 12\}.
    \item ExcelFormer: number of heads $\in$ \{4, 8, 16, 32\}.
    \item BRIDGE: layers of TabTransformer $\in$ \{2, 3, 4\}, layers of GCN $\in$: \{1, 2, 3, 4\}.
    \item RDL: layers of heterogeneous GraphSAGE $\in$: \{1, 2, 3, 4\}.
    \item \modelname: number of layers in linear attention $\in$ \{1, 2, 3, 4, 5\}, number of heads in linear attention $\in$ \{2, 4, 8, 16\}, number of layers in HGNN $\in$ \{1, 2, 3, 4, 5\}, initial value of learnable coefficient $\beta$ $\in$ \{0.5, 0.8\}.
\end{itemize}

\section{Pseudocode of \modelname}
\label{appendix:Pseudocode}
For clarity, we provide the pseudocode of the overall InRTL training procedure in Algorithm~\ref{alg:inrtl}.

\begin{algorithm}[H]
\caption{\modelname}
\label{alg:inrtl}
\begin{algorithmic}[1]
\Require Relational tables $\mathcal{T}$, PK-FK graph $\mathcal{G}$, number of layers $L$, ground-truth labels $y$, sampling function $\mathrm{Sampler()}$
\Ensure Trained \modelname model
\State Initialize the parameters of \modelname
\For{each training epoch}
  \ForAll{mini-batches $(\mathcal{T}_\mathcal{B}, \mathcal{G}_{\mathcal{B}}) \sim \mathrm{Sampler}(\mathcal{T}, \mathcal{G})$}
    \State Obtain $E_{\text{tgt}}$ and $E_{\text{aux}}$ for $\mathcal{T}_\mathcal{B}$ as described in Sec.~\ref{sec:methodology}
    \State $H^{(0)}_{\text{intra}} \gets E_{\text{tgt}}, \quad H^{(0)}_{\text{inter}} \gets E_{\text{tgt}}$
    \For{$l = 0$ to $L-1$}
      \State $Q \gets H^{(l)}_{\text{intra}}W_Q,\; K \gets H^{(l)}_{\text{intra}}W_K,\; V \gets H^{(l)}_{\text{intra}}W_V$
      \State Update $H^{(l+1)}_{\text{intra}}$ according to Eq.~\ref{eq:intra-final}
      \State Update $H^{(l+1)}_{\text{inter}}$ with $\mathcal{G}_{\mathcal{B}}$ according to Eq.~\ref{eq:inter-hgnn}
    \EndFor
    \State Compute prediction $\hat{y}_{\mathcal{B}}$ using Eq.~\ref{eq:fusion}
    \State Obtain mini-batch labels $y_{\mathcal{B}}$ from $y$
    \State Compute loss on $\hat{y}_{\mathcal{B}}$ and $y_{\mathcal{B}}$, and update parameters
  \EndFor
\EndFor
\State \Return Trained \modelname model
\end{algorithmic}
\end{algorithm}

\end{document}